\pdfoutput=1
\documentclass[10pt]{article}

\usepackage[conBIB,NumType]{stylefile_FR}
\usepackage{enumerate}
\usepackage{graphicx}
\usepackage{tikz-cd}

\title{Echoes of the Past: \\ A Unified Perspective on Fading memory and Echo States}
\author{Juan-Pablo Ortega\footnote{Division of Mathematical Sciences, School of Physical and Mathematical Sciences, Nanyang Technological University, Singapore} \qquad Florian Rossmannek$^*$}
\Headings{Echoes of the past: A unified perspective on fading memory and echo states}{Ortega and Rossmannek}

\begin{document}

\maketitle

\begin{abstract}
	Recurrent neural networks (RNNs) have become increasingly popular in information processing tasks involving time series and temporal data.
A fundamental property of RNNs is their ability to create reliable input/output responses, often linked to how the network handles its memory of the information it processed.
Various notions have been proposed to conceptualize the behavior of memory in RNNs, including steady states, echo states, state forgetting, input forgetting, and fading memory.
Although these notions are often used interchangeably, their precise relationships remain unclear.
This work aims to unify these notions in a common language, derive new implications and equivalences between them, and provide alternative proofs to some existing results.
By clarifying the relationships between these concepts, this research contributes to a deeper understanding of RNNs and their temporal information processing capabilities.
\end{abstract}


\section{Introduction}

Recurrent neural networks (RNNs) have become increasingly popular as machine learning models in tasks involving time series and temporal data generated by dynamical systems.
What makes RNNs so popular in applications is their efficient scalability and ease of training, owing to model designs such as reservoir computing (RC).
Other machine learning paradigms such as large language models require vast amounts of training data and cost-intensive computational resources for the initial training as well as retraining sessions, which bring increasingly difficult challenges in its wake pertaining to energy consumption and cost of chatbot queries \cite{BogmansEtal2025, JiangSonneEtal2024}.
In contrast, RC is designed so that only a small part of the model is being trained algorithmically, the main part of the model being chosen randomly \cite{Jaeger2010, MaassNatschMarkram2002} or even realized by a physical system.
The latter has sparked a particular interest in the physics and robotics community, which explore different systems that can function as these reservoirs.
The first example of such a physical reservoir was a bucket of water, whose computational ability resides in translating input sound waves to ripples of waves on the surface of the water \cite{FernandoSojakka2003}.
Since then many other physical reservoirs have been proposed from photonics and analog circuits to mechanical and biological bodies, to name only a few \cite{Nakajima2020, TanakaEtal2019}.
A particularly promising direction is the use of quantum systems as reservoirs for so-called quantum RC \cite{FujiiNakajima2017, GhoshEtal2021, MujalEtal2021, QRC1, QRC2, QRC5}.

On the theoretical side, a fundamental property of RNNs is their ability to create reliable input/output responses.
This ability is often linked to how the network handles the information it processes -- or rather, forgets it.
Over the past few decades, various notions have been proposed to conceptualize the decay of its memory.
Prominent keywords are steady states \cite{ChuaGreen1976, BoydChua1985}, echo states, state forgetting, input forgetting \cite{Jaeger2010}, and fading memory \cite{BoydChua1985}.
It is a long-standing folklore belief that these notions can be used interchangeably.
Although some subtleties concerning this belief have been identified in recent years \cite{Manjunath2020ProcA, RC30, RC31}, a deeper understanding of the relationship among these notions is lacking.
Such a task is difficult in part because the notions are not even defined using the same terminology.
This gap will be filled in the present work, in which we unify all notions and present new results linking them.

Broadly speaking, an RNN processes a sequence of inputs $(\seq{u}{t})_t$ and produces a sequence of states $(\seq{x}{t})_t$, where at time $t$ the current input $\seq{u}{t}$ and the previous state $\seq{x}{t-1}$ are transformed into the next state $\seq{x}{t}$ by the so-called state map $f$.
In applications, it is a crucial part of the model design that the states are further transformed by a readout map to yield the final output.
But adding a readout does not affect the underlying dynamics or its memory of the processed information and is not of relevance in this work.

There is a large spectrum of modeling situations in which, in the words of Volterra \cite{Volterra1959}, an \textit{``extremely natural postulate is to suppose that the influence of the (input) a long time before the given moment gradually fades out"}.
Such circumstances are pervasive in physical systems with friction or dissipation and in econometrics, where most parametric models are designed to exhibit some sort of fading memory-type feature \cite{Diebold2024, Hamilton1994}.
Similar ideas appear in cybernetics, where adaptation to a changing environment requires feedback controllers to rely on current signals rather than outdated ones \cite{Ashby1956, Wiener2019}.
This same principle underlies modern adaptive algorithms in computer science:
in online learning, estimators employ decaying step sizes or explicit forgetting mechanisms to remain responsive to non-stationary data streams \cite{ShalevShwartz2012};
reinforcement learning algorithms impose exponential decay through temporal-difference methods and eligibility traces \cite{SuttonBarto2018};
and streaming algorithms and adaptive filters downweight older observations to efficiently track evolving patterns \cite{DatarEtal2002, Haykin2014}.
Across these fields, emphasizing recent information is not merely a heuristic but a formally justified response to environments in which the data-generating process itself changes over time.

This emphasis on recent information is also well established and standard practice in control and system-identification \cite{Sepulchre2021}, particularly when the underlying system is subject to time variation or unmodeled disturbances.
Optimal state estimators such as the Kalman filter assign greater weight to new measurements through the Kalman gain, which increases when model uncertainty grows and thus emphasizes the informational value of recent data \cite{Maybeck1982, Simon2006}.
Adaptive control and recursive parameter estimation use exponential forgetting factors and finite-memory estimators to enable effective tracking of systems whose characteristics drift over time \cite{AstromWittenmark2008, Ljung1999}.

In all the situations that we just described, one is naturally led to systems that are `input forgetting' \cite{Jaeger2010, RC9}.
The states, in this case, have been coined `steady states' or `steady state solution' in earlier works \cite{ChuaGreen1976, BoydChua1985} and `echo states' in later works \cite{Jaeger2010}.
We caution that sometimes the terminology `echo states' includes the requirement that the sequence of states is unique for each given input sequence.

Even though the same modeling intuition underlies the notions of `input forgetting' and `fading memory', their mathematical formulations differ.
Customarily, fading memory is defined as continuous dependence of the echo states on the input sequence \cite{BoydChua1985, RC9}.
On first inspection, it may not be clear how continuous dependence captures the idea of fading memory.
The key to this definition is the choice of topology with respect to which continuity is required.
Indeed, some topologies on the space of input sequences boast the property that two sequences are close if their recent entries are close even if their distant past entries differ significantly.
Since there are several topologies with this property, there are different notions of fading memory depending on the choice of topology.
Although they are closely related to each other and sometimes overlap, various choices give rise to non-identical notions of fading memory \cite{RC30,RC31}.
The notions of input forgetting and fading memory are concerned with echo states of distinct input sequences that are similar to each other in the recent past.
On the contrary, the notion of `state forgetting' considers a single input sequence and two different initial states.
If the next states produced by the state map become increasingly similar to each other as time evolves regardless of the dissimilarity of the initial states, then the system is state forgetting \cite{Jaeger2010, RC9}.
The initial states may not be part of the sequence forming the input's echo states, but some form of attraction is at play, where the next states obtained from the initial ones approach the echo states.
Indeed, one of the overarching principles governing all these various notions is that the echo states form the dynamical attractor of the RNN to which the states converge \cite{Jaeger2010, CeniEtal2020PhysD, ManjunathJaeger2013, RC31}.

There had been a string of works that derived sufficient conditions for the existence of unique echo states, and it was observed that fading memory as well as various input and state forgetting properties held as well \cite{QRC1, RC7, RC9, RC28}.
However, in those results, all these properties were the consequence of a single sufficient criterion (always some form of contractivity of the state map), and it had therefore not been unveiled whether the notions imply each other independently of the common sufficient criterion.

In this work, we unify the various notions surveyed above in a common language, in which it will become easier to see how they are linked to forward and pullback attraction of the echo states.
We will derive new implications and equivalences between several notions, and along the way we will encounter alternative and shorter proofs to some of the classical results in the literature.
The main notions and statements are discussed in \cref{sec_main}, and they will be related to notions from the literature in \cref{sec_unifying}.

\section{Echoes of the past}
\label{sec_main}

\subsection{Mathematical framework}

In the following, we fix the notation and the underlying mathematical objects that will be used throughout.
We denote by $\Z$ the set of integers, by $\N$ the set of strictly positive integers, $\N_0 = \N \cup \{0\}$, and $\Z_- = \Z \backslash \N$.
Subsets of topological spaces and products of topological spaces are endowed with the subspace topology and the product topology, respectively.
Throughout, let $\Uc$ be a Hausdorff space and $(\Xc,d_{\Xc})$ be a metric space.
Sequences are denoted as underlined letters, e.g.\ $\Seq{u} = (\seq{u}{t})_{t \in \Z_-} \in \Uc^{\Z_-}$.
We denote the right-shift operator $(\seq{u}{t})_t \mapsto (\seq{u}{t-1})_t$ on all left- and bi-infinite sequence spaces by the letter $T$ and the left-shift operator $(\seq{u}{t})_t \mapsto (\seq{u}{t+1})_t$ on bi-infinite sequence spaces by $\sigma$.
Fix backwards shift-invariant subsets $\US^- \subseteq \Uc^{\Z_-}$ and $\XS^- \subseteq \Xc^{\Z_-}$, that is, $T^{-1}(\US^-) = \US^-$ and likewise for $\XS^-$, and fix a shift-invariant subset $\US \subseteq \Uc^{\Z}$ that is mapped surjectively onto $\US^-$ by the truncation, that is, $\sigma(\US) = \US$ and $\tau(\US) = \US^-$, where $\tau \colon \US \rightarrow \US^-$, $\Seq{u} \mapsto (\seq{u}{t})_{t \in \Z_-}$.
In addition, we assume that $\US$ contains all sequences of the form $(\dots,\seq{u}{-1}^-,\seq{u}{0}^-,\seq{u}{1},\seq{u}{2},\dots)$ with $\Seq{u}^- \in \US^-$ and $\Seq{u} \in \US$.

\subsection{Echo states and fading memory}
\label{sec_ESP_FMP}

Consider the state-space system governed by a continuous state map $f \colon \Xc \times \Uc \rightarrow \Xc$.
A {\bfi solution for the input} $\Seq{u} \in \US^-$ is an element $\Seq{x} \in \XS^-$ that satisfies $\seq{x}{t} = f(\seq{x}{t-1},\seq{u}{t})$ for all $t \in \Z_-$.
We also call the pair $(\Seq{x},\Seq{u})$ a {\bfi solution}.
The set of all solutions will be denoted
\begin{equation*}
	\Sc
	= \{ (\Seq{x},\Seq{u}) \in \XS^- \times \US^- \colon \seq{x}{t} = f(\seq{x}{t-1},\seq{u}{t}) \text{ for all } t \in \Z_- \}.
\end{equation*}
We say that the state-space system has the {\bfi echo state property (ESP)} if there exists a unique solution $\Seq{x} \in \XS^-$ for any given input $\Seq{u} \in \US^-$.
In the literature, this definition of the ESP is sometimes called the {\bfi unique solution property} \cite{Manjunath2022Nonlin}.
As mentioned in the introduction, in the presence of the ESP, we say that the state-space system has the {\bfi fading memory property (FMP)} if the map $\US^- \rightarrow \XS^-$ that associates to any given input its unique solution is continuous.
The FMP can be defined without presupposing the ESP \cite{Manjunath2020ProcA, RC31} but that generalization is not of relevance here.
We discussed in the introduction that different choices of topologies on the sequence spaces give rise to different notions of FMP.
In this work, we use the product topologies.
The resulting FMP was termed {\bfi product FMP} in \cite{RC30}.
The FMP originally considered in \cite{BoydChua1985} was based on the topology induced by a weighted sup-norm.
Since the product topology is at least as coarse as the topology induced by a weighted sup-norm \cite{RC7}, the product FMP always implies the FMP in \cite{BoydChua1985}.
In that sense, the product FMP is the strongest notion of FMP encountered in the literature (topologies coarser than the product topology make coordinate projections discontinuous and are therefore never used).
Restricting our attention to the strongest FMP is satisfactory for the results we present in this work.
For detailed discussions and results about the different FMPs, we refer the reader to \cite{RC9, RC30, RC31}.
Lastly, we point out that compactness of $\XS^-$ or $\US$ with respect to the product topology, which we encounter later as assumptions, is a weaker assumption than compactness with respect to a finer topology.
Thus, by choosing the product topology, we get the strongest FMP and the weakest assumptions later.

\subsection{State and input forgetting}
\label{sec_forgetting}

Consider the family of functions $\psi_n \colon \Xc \times \US \rightarrow \Xc$, $n \in \N$, defined recursively by $\psi_1(x,\Seq{u}) = f(x,\seq{u}{1})$ and $\psi_n(x,\Seq{u}) = f(\psi_{n-1}(x,\Seq{u}),\seq{u}{n})$.
A variation of this family of functions has appeared in \cite{RC9} under the name `reservoir flow'.
We will define state and input forgetting properties based on the relations
\begin{equation}
\label{FP}
	\limsup_{n \rightarrow \infty} d_{\Xc}( \psi_n(x,\Seq{u}) , \psi_n(x',\Seq{u}) )
	= 0
\end{equation}
and
\begin{equation}
\label{sFP}
	\limsup_{n \rightarrow \infty} d_{\Xc}( \psi_n(x,T^n(\Seq{u})) , \psi_n(x',T^n(\Seq{u})) )
	= 0,
\end{equation}
respectively.
A state is called {\bfi reachable} if it appears as the time zero component of some solution.
We denote the set of all reachable states by
\begin{equation*}
	\Rc
	= \{ x \in \Xc \colon \text{there exists } (\Seq{x},\Seq{u}) \in \Sc \text{ with } \seq{x}{0} = x \}
\end{equation*}
and point out that $\psi_n(\Rc \times \US) = \Rc$ for all $n \in \N$.

\begin{definition}
	We say that the state-space system has
\begin{enumerate}[\upshape (i)]\itemsep=0em
\item
... the {\bfi state forgetting property (SFP)} if \eqref{FP} holds for all $\Seq{u} \in \US$ and $x,x' \in \Xc$.

\item
... the {\bfi input forgetting property (IFP)} if \eqref{FP} holds for all $\Seq{u} \in \US$ and $x,x' \in \Rc$.

\item
... the {\bfi shifted state forgetting property (s-SFP)} if \eqref{sFP} holds for all $\Seq{u} \in \US$ and $x,x' \in \Xc$.

\item
... the {\bfi shifted input forgetting property (s-IFP)} if \eqref{sFP} holds for all $\Seq{u} \in \US$ and $x,x' \in \Rc$.
\end{enumerate}
We add the attribute {\bfi state-uniform (uniform)} if the convergence of the limit superior is uniform in $x,x'$ (in $\Seq{u},x,x'$).
\end{definition}

\begin{remark}
\label{rem_SFP_IFP}
	The following implications and equivalences hold by definition.
\begin{center}
%
\begin{tikzcd}[row sep = 2em, column sep = 2em]
	\text{SFP} \arrow{d} \pgfmatrixnextcell \arrow{l} \text{state-unif.\ SFP} \arrow{d} \pgfmatrixnextcell \arrow{l} \text{unif.\ SFP} \arrow{d} \arrow{r} \pgfmatrixnextcell \arrow{l} \text{unif.\ s-SFP} \arrow{d} \arrow{r} \pgfmatrixnextcell \text{state-unif.\ s-SFP} \arrow{d} \arrow{r} \pgfmatrixnextcell \text{s-SFP} \arrow{d}
	\\
	\text{IFP} \pgfmatrixnextcell \arrow{l} \text{state-unif.\ IFP} \pgfmatrixnextcell \arrow{l} \text{unif.\ IFP} \arrow{r} \pgfmatrixnextcell \arrow{l} \text{unif.\ s-IFP} \arrow{r} \pgfmatrixnextcell \text{state-unif.\ s-IFP} \arrow{r} \pgfmatrixnextcell \text{s-IFP}
\end{tikzcd}%
\end{center}
\end{remark}

From the given definition, it may not be immediately apparent why the IFP indeed conceptualizes `input forgetting'.
To see how the IFP is motivated, consider $x^1,x^2 \in \Rc$.
Take solutions $(\Seq{x}^1,\Seq{u}^1), (\Seq{x}^2,\Seq{u}^2) \in \XS^- \times \US^-$ with $\seq{x}{0}^1 = x^1$ and $\seq{x}{0}^2 = x^2$.
Then, $\psi_n(x^i,\Seq{u})$ is exactly the time zero component of a solution for the input $\Seq{u}^{i,n} := (\dots,\seq{u}{-1}^i,\seq{u}{0}^i,\seq{u}{1},\dots,\seq{u}{n})$.
The $n$ most recent inputs in the sequences $\Seq{u}^{1,n}$ and $\Seq{u}^{2,n}$ are the same.
Thus, although the IFP is defined through states $x^1,x^2 \in \Rc$, it really considers solutions of different input sequences whose difference lies further and further in the past as $n$ grows; see also \cref{fig_FPs}.

\begin{figure*}
\includegraphics[width=\textwidth]{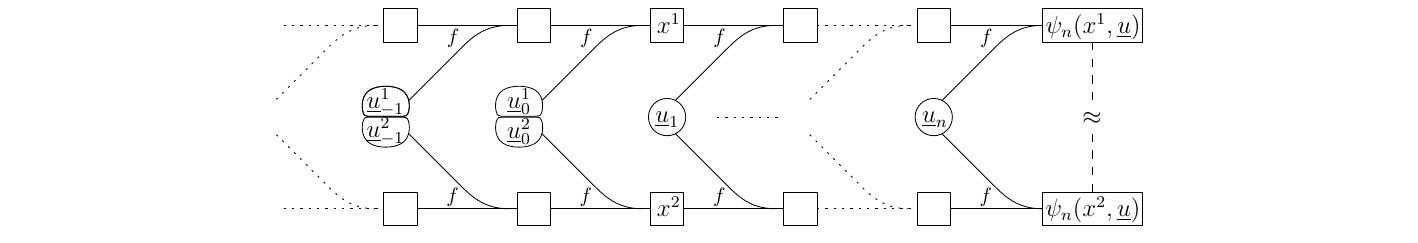}
\includegraphics[width=\textwidth]{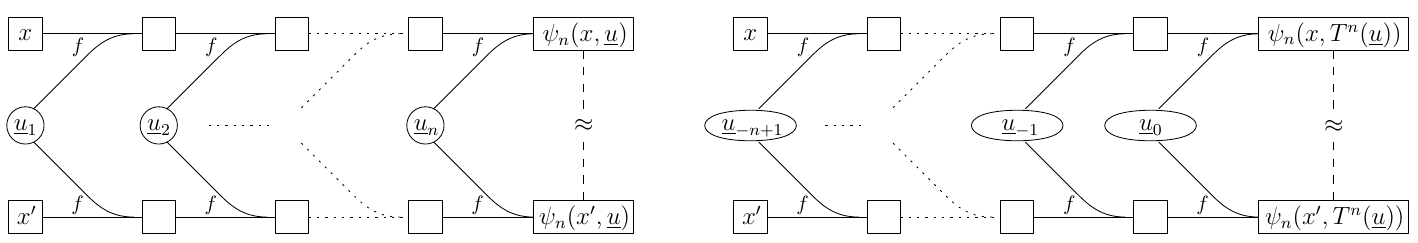}
\vspace*{-25pt}
\caption{\label{fig_FPs} Illustrations of the IFP (top), SFP (left), and s-SFP (right).}
\end{figure*}

\subsection{Attraction}

We hinted in the introduction that one of the overarching principles is that the echo states form the dynamical attractor of the RNN to which the states converge.
By the invariance $\psi_n(\Rc \times \US) = \Rc$, the SFP yields for all $\Seq{u} \in \US$ and $x \in \Xc$\footnote{
Here, $\mathrm{dist}_{\Xc}(x,A) = \inf_{y \in A} d_{\Xc}(x,y)$ for a point $x \in \Xc$ and a subset $A \subseteq \Xc$.}
\begin{equation*}
	\limsup_{n \rightarrow \infty} \mathrm{dist}_{\Xc}( \psi_n(x,\Seq{u}) , \Rc )
	\leq \limsup_{n \rightarrow \infty} d_{\Xc}( \psi_n(x,\Seq{u}) , \psi_n(x',\Seq{u}) )
	= 0,
\end{equation*}
where $x'$ is an arbitrary element in $\Rc$.
Likewise, the s-SFP yields
\begin{equation*}
	\limsup_{n \rightarrow \infty} \mathrm{dist}_{\Xc}( \psi_n(x,T^n(\Seq{u})) , \Rc )
	\leq \limsup_{n \rightarrow \infty} d_{\Xc}( \psi_n(x,T^n(\Seq{u})) , \psi_n(x',T^n(\Seq{u})) )
	= 0.
\end{equation*}
Although both the SFP and the s-SFP imply attraction of $\Rc$, we now argue that the SFP captures forward attraction and the s-SFP pullback attraction.
To make this claim rigorous, let us specify the precise dynamical system.
Despite being intrinsically non-autonomous, it was proposed in \cite{RC31} to model the dynamics autonomously on an extended sequence space.
To this end, consider the map $\varphi \colon \XS^- \times \US \rightarrow \XS^- \times \US$ given by $\varphi(\Seq{x},\Seq{u}) = ((\Seq{x},f(\seq{x}{0},\seq{u}{1})),\sigma(\Seq{u}))$, whose iterates can be seen as extensions of the maps $\psi_n$. 
Concretely, $\psi_n \circ (p_0 \times \pi) = p_0 \circ \varphi^n$, where $p_0 \colon \XS^- \times \US \rightarrow \Xc$ and $\pi \colon \XS^- \times \US \rightarrow \US$ are the projections $p_0(\Seq{x},\Seq{u}) = \seq{x}{0}$ and $\pi(\Seq{x},\Seq{u}) = \Seq{u}$.
It has been shown in \cite{RC31} that the global attractor $\Ac := \bigcap_{n \in \N} \varphi^n(\XS^- \times \US)$ of $\varphi$ is exactly equal to the inverse image $\tau^{-1}(\Sc)$ of the set of solutions under the truncation $\tau \colon \XS^- \times \US \rightarrow \XS^- \times \US^-$ extended to the product space.
In particular, the set of reachable states $\Rc = p_0(\Ac)$ is the $p_0$-projection of the attractor $\Ac$.
The fibers $\Ac \cap \pi^{-1}(\Seq{u})$ of the attractor capture pullback attraction, and the fibers $\Ac \cap \pi^{-1}(\sigma^n(\Seq{u}))$ capture forward attraction.
The relation $\psi_n \circ (p_0 \times \pi) = p_0 \circ \varphi^n$ implies that $\psi_n(x,\Seq{u})$ belongs to the $p_0$-projection of the fiber $\pi^{-1}(\sigma^n(\Seq{u}))$, and $\psi_n(x,T^n(\Seq{u}))$ belongs to the $p_0$-projection of the fiber $\pi^{-1}(\Seq{u})$, which together with the equality $\Rc = p_0(\Ac)$ finally explains the association of the SFP with forward and of the s-SFP with pullback attraction.\footnote{
One can also consider the attractor $\Ac_0 := \bigcap_{n \in \N} \psi_n(\Xc \times \US)$ of the non-autonomous dynamics of $\psi_n$.
But since the inclusion $p_0(\Ac) \subseteq \Ac_0$ is strict in general, the attractor $\Ac$ gives a more detailed picture.}

\begin{remark}
\label{rem_DSobs_inputs}
	If the inputs are known to be generated by a deterministic dynamical system $\phi \colon \Mc \rightarrow \Mc$, transformed by an observation function $\omega \colon \Mc \rightarrow \Uc$, then instead of the maps $\psi_n$ defined on $\Xc \times \US$ one would consider the alternative maps $\hat{\psi}_n \colon \Xc \times \Mc \rightarrow \Mc$ defined by $\hat{\psi}_1(x,p) = f(x,\omega \circ \phi(p))$ and $\hat{\psi}_n(x,p) = f(\hat{\psi}_{n-1}(x,p),\omega \circ \phi^n(p))$.
However, the (shifted) SFP and IFP defined through these maps $\hat{\psi}_n$ are equivalent to the (shifted) SFP and IFP defined through $\psi_n$ with $\US = \{ (\omega(\phi^t(p)))_{t \in \Z} \colon p \in \Mc \}$.
Our setup working with $\psi_n$ on $\Xc \times \US$ covers the most general case of possible inputs.
\end{remark}

\subsection{Main result}

In our main result, stated below, we extend the diagram in \cref{rem_SFP_IFP} with additional implication arrows under compactness assumptions.
It generalizes \cite[Proposition 1]{Jaeger2010}, \cite[Theorem 1]{Manjunath2022Nonlin}, and \cite[Remark 1]{KobayashiTranNakajima2024}, and provides alternative and shorter proofs.

\begin{theorem}
\label{thrm_main}
	Consider the diagram below.
If $\XS^-$ is compact, then the implication arrows indicated in blue are valid;
and if $\XS^-$ and $\US$ are compact, then the implication arrows indicated in green are valid.
Furthermore, we will see counterexamples to the implication arrows indicated in purple.
\begin{center}
%
\begin{tikzcd}[row sep = 2em, column sep = 2em]
	\text{SFP} \arrow[purple, shift right=1, "/"{anchor=center,sloped}]{r} \arrow{d} \pgfmatrixnextcell \arrow[shift right=1]{l} \text{state-unif.\ SFP} \arrow{d} \pgfmatrixnextcell \arrow{l} \text{unif.\ SFP} \arrow[shift right=1]{d} \arrow{r} \pgfmatrixnextcell \arrow{l} \text{unif.\ s-SFP} \arrow[shift right=1]{d} \arrow[shift right=1]{r} \pgfmatrixnextcell \arrow[green, shift right=1]{l} \text{state-unif.\ s-SFP} \arrow[shift right=1]{d} \arrow[shift right=1]{r} \pgfmatrixnextcell \arrow[purple, shift right=1, "/"{anchor=center,sloped}]{l} \text{s-SFP} \arrow{d}
	\\
	\text{IFP} \arrow[purple, shift right=1, "/"{anchor=center,sloped}]{r} \pgfmatrixnextcell \arrow[shift right=1]{l} \text{state-unif.\ IFP} \pgfmatrixnextcell \arrow{l} \text{unif.\ IFP} \arrow[green, shift right=1]{u} \arrow{r} \pgfmatrixnextcell \arrow{l} \text{unif.\ s-IFP} \arrow[green, shift right=1]{u} \arrow[shift right=1]{r} \pgfmatrixnextcell \arrow[green, shift right=1]{l} \text{state-unif.\ s-IFP} \arrow[shift right=1]{r} \arrow[blue, shift right=1]{d} \arrow[blue, shift right=1]{u} \pgfmatrixnextcell \arrow[purple, shift right=1, "/"{anchor=center,sloped}]{l} \text{s-IFP}
	\\
	\pgfmatrixnextcell \pgfmatrixnextcell \pgfmatrixnextcell \pgfmatrixnextcell \text{ESP} \arrow[blue, shift right=1]{u} \arrow[blue]{r} \pgfmatrixnextcell \text{FMP}
\end{tikzcd}%
\end{center}
\end{theorem}

\begin{remark}
	We point out that $\XS^-$ is compact if and only if $\Xc$ is compact and $\XS^-$ is a closed subset of $\Xc^{\Z_-}$.
Likewise, $\US$ is compact if and only if $\Uc$ is compact and $\US$ is a closed subset of $\Uc^{\Z}$.
\end{remark}

\subsection{Technical details}

Let us prove \cref{thrm_main}.
That the ESP implies the FMP if $\XS^-$ is compact has been proved in \cite{Manjunath2020ProcA, RC31}.
In fact, we mentioned in \cref{sec_ESP_FMP} that fading memory can also be defined without presupposing the ESP, which involves generalizing continuity of the solution map to a set-valued solution map.
In this case, the reverse implication -- that the FMP implies the ESP -- holds as soon as there is some input for which there exists exactly one solution \cite{Manjunath2020ProcA, RC31}.
All remaining implications in \cref{thrm_main} follow from the next two lemmas.

\begin{lemma}
\label{IFP_ESP}
	Suppose the state-space system has the state-uniform s-IFP.
Then, there is at most one solution for each input.
In particular, if $\XS^-$ is compact, then the state-space system has the ESP.
\end{lemma}

\begin{proof}
	Suppose $\Seq{x},\Seq{x}' \in \XS^-$ are two solutions for a given input $\Seq{u}^- \in \US^-$.
Fix $t \in \Z_-$.
Take any $\Seq{u} \in \tau^{-1}(\Seq{u}^-)$ and note that $\psi_n(\seq{x}{t-n},T^{n-t}(\Seq{u})) = \seq{x}{t}$.
By the s-IFP applied to $T^{-t}(\Seq{u})$,
\begin{equation*}
	d_{\Xc}( \seq{x}{t} , \seq{x}{t}' )
	= d_{\Xc}( \psi_n(\seq{x}{t-n},T^{n-t}(\Seq{u})) , \psi_n(\seq{x}{t-n}',T^{n-t}(\Seq{u})) )
\end{equation*}
converges to zero since the convergence is state-uniform.
Thus, $\Seq{x} = \Seq{x}'$.
If $\XS^-$ is compact, then the fibers of the attractor $\Ac$ are non-empty by a standard topology argument.
Indeed, each $\Ac \cap \pi^{-1}(\Seq{u})$ becomes a nested intersection of the non-empty compact sets $\varphi^n( \XS^- \times \pi^{-1}(T^n(\Seq{u})) )$.
The existence of at least one solution for each input is equivalent to $\Ac$ having non-empty fibers.
\end{proof}

\begin{lemma}
\label{lem_ESP_SFP}
	Suppose the state-space system has the ESP.
\begin{enumerate}[\upshape (i)]\itemsep=0em
\item
If $\XS^-$ is compact, then the state-space system has the state-uniform s-SFP.
\item
If $\US$ and $\XS^-$ are compact, then the state-space system has the uniform s-SFP.
\end{enumerate}
\end{lemma}

\begin{proof}
	(ii)
Suppose for contradiction the uniform s-SFP does not hold, that is, there exists an $\epsilon > 0$, a strictly increasing sequence $(n_k)_k \subseteq \N$, and sequences $(\Seq{u}^k)_k \subseteq \US$, $(x_k)_k,(x_k')_k \subseteq \Xc$ with
\begin{equation*}
	d_{\Xc}( \psi_{n_k}(x_k,T^{n_k}(\Seq{u}^k)) , \psi_{n_k}(x_k',T^{n_k}(\Seq{u}^k)) )
	\geq \epsilon.
\end{equation*}
Take $\Seq{x}^k,\Seq{\hat{x}}^k \in \XS^-$ with $\seq{x}{0}^k = x_k$ and $\seq{\hat{x}}{0}^k = x_k'$.
By compactness, $\varphi^{n_k}(\Seq{x}^k,T^{n_k}(\Seq{u}^k))$ and $\varphi^{n_k}(\Seq{\hat{x}}^k,T^{n_k}(\Seq{u}^k))$ have some accumulation points $(\Seq{x},\Seq{u})$ and $(\Seq{x}',\Seq{u}')$ satisfying $\Seq{u} = \Seq{u}'$.
Since these accumulation points necessarily belong to $\Ac$, the ESP implies $\Seq{x} = \Seq{x}'$.
However, this contradicts $d_{\Xc}( p_0(\Seq{x}) , p_0(\Seq{x}') ) \geq \epsilon$, which follows from the relation $\psi_n \circ (p_0 \circ \pi) = p_0 \circ \varphi^n$.

\noindent (i)
The same argument as for (ii) with $\Seq{u}^k = \Seq{u}$ not depending on $k$ yields the state-uniform s-SFP since then $\varphi^{n_k}(\Seq{x}^k,T^{n_k}(\Seq{u})) \in \XS^- \times \{\Seq{u}\}$ still belongs to a compact space and any accumulation point thereof must belong to $\Ac \cap \pi^{-1}(\Seq{u})$.
\end{proof}

The following example shows that the implication arrows in \cref{thrm_main} indicated in purple do not hold in general, even if the underlying spaces are compact and connected.

\begin{example}
	Consider the homeomorphism $x \mapsto x^2$ of the unit interval $[0,1]$.
Since the endpoints of the interval are fixed points, this map induces a homeomorphism $g \colon S^1 \rightarrow S^1$ of the unit circle.
Note that $g^n(x)$ converges to 0 as $n \rightarrow \infty$ for all $x \in S^1$.
Thus, the state-space system governed by $f \colon S^1 \times \Uc \rightarrow S^1$, $(x,u) \mapsto g(x)$ has the SFP.
On the other hand, the SFP is not state-uniform.
Every state is reachable since $g$ is a homeomorphism, that is, $\Rc = S^1$.
This implies that the ESP does not hold and that the IFP coincides with the SFP.
Lastly, if we take $\US$ to contain only constant sequences, then the s-SFP and s-IFP also coincide with the SFP.
We remark that the state map can be perturbed to no longer be input-independent but retain all the other properties of $f$.
\end{example}

\begin{remark}
	If the inputs are observations of a hidden dynamical system as in \cref{rem_DSobs_inputs}, then the maps $\psi_n(x,T^n(\Seq{u}))$ recover the `echo state family' considered in \cite{HartHookDawes2020}.
Therein, it was shown that this echo state family converges under a contractivity condition on the state map, which guarantees the ESP.
Since \cite{HartHookDawes2020} assumes compactness, it follows from our results that the contractivity condition is not necessary.
Indeed, as soon as the ESP holds, the s-SFP implies convergence of the echo state family.
\end{remark}

\subsection{The significance of forgetting states}

The SFP bears particular significance for applications.
Suppose we are trying to learn an unknown invertible dynamical system as in \cref{rem_DSobs_inputs}, that is, the inputs are of the form $\seq{u}{t} = \omega(\phi^t(p))$, and at the core of our learning method is the state-space system $f$.
For example, $f$ could be an echo state network or the update map of a reservoir computer \cite{Jaeger2010}.
On an abstract level, successful learning manifests in the existence of a semi-conjugacy $\zeta \colon \Mc \rightarrow \Xc$, (meaning $f(\zeta(p),\omega(\phi(p))) = \zeta(\phi(p))$ for all $p \in \Mc$) that, moreover, is an embedding.
Such a semi-conjugacy is called a generalized synchronization (GS) \cite{KocarevParlitz1996, LuHuntOtt2018Chaos}.
A prerequisite for the existence of a GS is the ESP of the system.
If the GS is an embedding, then it is, in fact, a homeomorphism from $\Mc$ to $\Rc$ \cite{RC31} that, under certain hypotheses, becomes differentiable \cite{RC18}.
A trajectory $(\phi^t(p))_{t \in \Z_-}$ is synchronized through $\zeta$ with the unique solution $\Seq{x}$ of the state-space system for the input $\Seq{u} = (\omega(\phi^t(p)))_{t \in \Z_-}$.
If we have access to the state $\seq{x}{t}$ at time $t$, then applying the transformation $h := \omega \circ \phi \circ \zeta^{-1}$ to $\seq{x}{t}$ yields the next observation $\seq{u}{t+1}$, that is, a perfect prediction.
In practice, the map $h$ will be learned from data, giving rise to an approximate prediction $\hat{h}(\seq{x}{t}) \approx h(\seq{x}{t}) = \seq{u}{t+1}$ \cite{RC26}.

This approximate prediction requires access to the state $\seq{x}{t}$.
The solution $\Seq{x}$ is an object indexed in the infinite past.
However, the actual implementation of $f$ begins at some point in time, say 0.
Since we do not know the solution $\Seq{x}$, we cannot initialize the system at $\seq{x}{0}$, but instead rely on a random initialization or an educated guess $\seq{x}{0}'$ \cite{RC26}.
Since the initial state $\seq{x}{0}'$ contains imperfect information about the dynamics we are trying to learn, it is the SFP that guarantees that $\seq{x}{t}'$, the actual state of the implemented system at time $t$, will be close to $\seq{x}{t}$ after a sufficient burn-in time after which $\hat{h}(\seq{x}{t}') \approx \hat{h}(\seq{x}{t}) \approx h(\seq{x}{t}) = \seq{u}{t+1}$.
In the language of \cite{AmigoEtal2024}, the SFP ensures that the implemented system is asymptotically synchronized.

\section{Unifying notions}
\label{sec_unifying}

To show that the newly presented definitions of state and input forgetting unify previous notions in the literature, we discuss the relevant definitions found in \cite{BoydChua1985, Jaeger2010, Manjunath2022Nonlin, RC9, TranRuefferKellett2019TAC}.

\subsection{Contracting and forgetting {\`a} la Jaeger}
\label{sec_Ja}

We point out that reachable states are what Jaeger \cite{Jaeger2010} called `end-compatible'.
More specifically, a state $x \in \Xc$ is {\bfi end-compatible} with an input sequence $\Seq{u} \in \US^-$ if there exists a solution $\Seq{x} \in \XS^-$ for the input $\Seq{u}$ with $\seq{x}{0} = x$.
The notions of state contracting, state forgetting, and input forgetting in \cite{Jaeger2010} are formulated with left- and right-infinite input sequences.
Although the maps $\psi_n$ formally take bi-infinite sequences as input, $\psi_n(x,\Seq{u})$ depends only on $\seq{u}{1},\dots,\seq{u}{n}$ and $\psi_n(x,T^n(\Seq{u}))$ depends only on $\seq{u}{-n+1},\dots,\seq{u}{0}$.
This allows us to express Jaeger's definitions of state contracting and state forgetting with the maps $\psi_n$.
To state his definition of input forgetting, let $\gamma^n \colon \US^- \times \US^- \rightarrow \US^-$, $n \in \N$, be the maps $\gamma^n(\Seq{u}',\Seq{u}) = (\dots,\seq{u}{-1}',\seq{u}{0}',\seq{u}{-n+1},\dots,\seq{u}{0})$.
Now, in the language of \cite{Jaeger2010}, the system is
\begin{itemize}
\item
state contracting if for every $\Seq{u} \in \US$ there exists some null-sequence $(\delta_n)_{n \in \N} \subseteq (0,1)$ such that for all $n \in \N$ and all states $x,x' \in \Xc$ it holds that $d_{\Xc}(\psi_n(x,\Seq{u}),\psi_n(x',\Seq{u})) < \delta_n$.
If the null-sequence does not depend on $\Seq{u}$, then the system is uniformly state contracting.
Unpacking the definition of the limit superior in \eqref{FP}, it is clear that this notion of state contracting is precisely the state-uniform SFP, and uniformly state contracting is exactly the uniform SFP.

\item
state forgetting if for every $\Seq{u} \in \US$ there exists some null-sequence $(\delta_n)_{n \in \N} \subseteq (0,1)$ such that for all $n \in \N$ and all states $x,x' \in \Xc$ it holds that $d_{\Xc}(\psi_n(x,T^n(\Seq{u})),\psi_n(x',T^n(\Seq{u}))) < \delta_n$.
As before, unpacking the definition of the limit superior in \eqref{sFP}, it is clear that this notion of state forgetting is precisely the state-uniform s-SFP.

\item
input forgetting if for every $\Seq{u}^0 \in \US$ there exists some null-sequence $(\delta_n)_{n \in \N} \subseteq (0,1)$ such that for all $n \in \N$, all $\Seq{u},\Seq{u}' \in \US^-$, and all states $x,x' \in \Xc$ that are end-compatible with $\gamma^n(\Seq{u},\tau(\Seq{u}^0))$ and $\gamma^n(\Seq{u}',\tau(\Seq{u}^0))$, respectively, it holds that $d_{\Xc}(x,x') < \delta_n$.
Note that $x$ being end-compatible with $\gamma^n(\Seq{u},\tau(\Seq{u}^0))$ is equivalent to the existence of a solution $\Seq{x}$ for the input $\Seq{u}$ such that $x = \psi_n(\seq{x}{0},T^n(\Seq{u}^0))$.
In particular, $\seq{x}{0} \in \Rc$ is reachable.
Thus, being input forgetting is equivalent to:
for every $\Seq{u}^0 \in \US$ there exists some null-sequence $(\delta_n)_{n \in \N} \subseteq (0,1)$ such that for all $n \in \N$ and all $x,x' \in \Rc$ it holds that $d_{\Xc}(\psi_n(x_0,T^n(\Seq{u}^0)) , \psi_n(x_0',T^n(\Seq{u}^0))) < \delta_n$.
This shows that this notion of input forgetting is precisely the state-uniform s-IFP.
\end{itemize}

\subsection{Forgetting {\`a} la Grigoryeva and Ortega}
\label{sec_GrOr}

The (uniform) state forgetting property introduced by Grigoryeva and Ortega in \cite{RC9} is exactly the (uniform) SFP we introduced in \cref{sec_forgetting}.
Their variant of the input forgetting property is formulated with the {\bfi induced functional}, which, assuming the state-space system has the ESP, is the unique map $\Hc \colon \US^- \rightarrow \Xc$ that assigns to every input sequence the time zero component of its unique solution, that is, $\Hc(\Seq{u}) = \seq{x}{0}$ for all $(\Seq{x},\Seq{u}) \in \Sc$.
In \cite{RC9}, the state-space system is said to have the (uniform) input forgetting property if for all $\Seq{u},\Seq{u}' \in \US$ with $\seq{u}{t} = \seq{u}{t}'$, $t \in \N$, (uniformly in $\Seq{u},\Seq{u}'$) $\lim_{n \rightarrow \infty} d_{\Xc}(\Hc \circ \tau \circ \sigma^n(\Seq{u}) , \Hc \circ \tau \circ \sigma^n(\Seq{u}')) = 0$.
Note that $\Hc \circ \tau \circ \sigma^n(\Seq{u}) = \psi_n(\Hc(\tau(\Seq{u})),\Seq{u})$ and $\Hc \circ \tau \circ \sigma^n(\Seq{u}') = \psi_n(\Hc(\tau(\Seq{u}')),\Seq{u})$.
The state $x' := \Hc(\tau(\Seq{u}'))$ is reachable.
Thus, the (uniform) input forgetting property in the sense of \cite{RC9} is equivalent to:
for all $\Seq{u} \in \US$ and all $x' \in \Rc$ (uniformly in $\Seq{u},x'$) $\lim_{n \rightarrow \infty} d_{\Xc}(\psi_n(\Hc(\tau(\Seq{u})),\Seq{u}) , \psi_n(x',\Seq{u})) = 0$.
By the triangle inequality, this is further equivalent to:
for all $\Seq{u} \in \US$ and all $x,x' \in \Rc$ (uniformly in $\Seq{u},x,x'$) $\lim_{n \rightarrow \infty} d_{\Xc}(\psi_n(x,\Seq{u}) , \psi_n(x',\Seq{u})) = 0$.
This is precisely the (uniform) IFP we introduced in \cref{sec_forgetting}.

\subsection{Attraction {\`a} la Manjunath}

Attractors have long played a central role in dynamical systems theory and come in many flavors \cite{Milnor1985}.
Here, we review the choice of definition made in Manjunath's work \cite{Manjunath2022Nonlin}, which discussed links between echo states, fading memory, and attractors, and is highly relevant to our discussion.
In \cite{Manjunath2022Nonlin}, the state-space system is said to have the {\bfi uniform attracting property} if for every $\Seq{u} \in \US$ there exists a solution $\Seq{x}(\Seq{u}) \in \XS^-$ for the input $\tau(\Seq{u})$ such that
\begin{equation*}
\begin{split}
	\lim_{n \rightarrow \infty} \sup_{t \in \Z_-} \sup_{\Seq{u} \in \US} \sup_{x \in \Xc} d_{\Xc}\left( \psi_n(x,T^{n-t}(\Seq{u})) , \seq{x}{t}(\Seq{u}) \right)
	= 0.
\end{split}
\end{equation*}
Note that if the uniform attracting property holds, then the ESP must hold and $\seq{x}{t}(\Seq{u}) = \Hc \circ \tau \circ T^{-t}(\Seq{u}) = \psi_n(\Hc \circ \tau \circ T^{n-t}(\Seq{u}),T^{n-t}(\Seq{u}))$.
Since $\US$ is shift-invariant, the supremum over $\Seq{u} \in \US$ can be substituted by the supremum over $T^{n-t}(\Seq{u}) \in \US$.
Thus, the uniform attracting property is equivalent to having the ESP and
\begin{equation*}
\begin{split}
	\lim_{n \rightarrow \infty} \sup_{\Seq{u} \in \US} \sup_{x \in \Xc} d_{\Xc}\left( \psi_n(x,\Seq{u}) , \psi_n(\Hc \circ \tau(\Seq{u}),\Seq{u}) \right)
	= 0.
\end{split}
\end{equation*}
From this it is clear that the ESP and uniform SFP together imply the uniform attracting property, and the reverse implication also holds by the triangle inequality.

\subsection{Steady states {\`a} la Boyd, Chua, and Green}

The notion of steady-states or steady-state solutions dates back to the classical works in linear filtering and control \cite{Kalman1960,KalmanBucy1961} but the precise definitions vary between different sources and earlier works preferred the terminology of asymptotic stability \cite{KalmanBertram1960a,KalmanBertram1960b}.
Here, we use the definition from the work of Chua and Green \cite{ChuaGreen1976} and of Boyd and Chua \cite{BoydChua1985} due to the high impact that they had on the field, partly because they were the first to establish a formal link to fading memory.
Their definitions are stated for systems in continuous time but are easily translated to a discrete time setting.
It should be noted that \cite{ChuaGreen1976,BoydChua1985} only consider bounded solutions.
Thus, to recover their definitions exactly we assume $\Xc$ to be bounded in this subsection.
Steady states are defined through asymptotic convergence of solutions in forward time.
When the state equation is considered in forward time only, solutions are uniquely determined by the maps $\psi_n$ once the initial condition at time zero is fixed.
Now, the state-space system is said to have the {\bfi unique steady-state property} if for all $\Seq{u} \in \US$ and all reachable states $x,x' \in \Rc$ it holds that $\lim_{n \rightarrow \infty} d_{\Xc}(\psi_n(x,\Seq{u}),\psi_n(x',\Seq{u})) = 0$.
This is precisely the IFP.

\subsection{Incremental stability {\`a} la Demidovich}

Whereas Boyd, Chua, and Green used steady states in relation to fading memory and Volterra series approximations, another strand of control-theory literature considered them in the context of Lyapunov methods \cite{Angeli2002}.
The foundation for this is sometimes attributed to the early work of Demidovich \cite{EfimovEtal1978, LevitanPapush1966, PavlovEtal2004}, but the topic gained greater traction in the past three decades.
In this line of work, a central notion linking steady states and Lyapunov methods is incremental stability.
Originally treated in the continuous time case, it has recently appeared in the discrete time setup as well \cite{JungersShakibWouw2024, PavlovWouw2012, TranRuefferKellett2019TAC}.
Incremental stability is almost the same as the SFP but is routinely defined through the notion of a function of class $\Lc$.
We recall that a function $\beta \colon [0,\infty)^2 \rightarrow [0,\infty)$ is of class $\Lc$ if, for any fixed $s \in [0,\infty)$, the function $\beta(s,\cdot)$ is continuous, strictly decreasing, and satisfies $\limsup_{t \rightarrow \infty} \beta(s,t) = 0$.
The state-space system is said to be {\bfi asymptotically incrementally stable (AIS)} if for all $\Seq{u} \in \US$ and $x,x' \in \Xc$ there exists a function $\beta$ of class $\Lc$ such that for all $n \in \N$
\begin{equation*}
	d_{\Xc}( \psi_n(x,\Seq{u}) , \psi_n(x',\Seq{u}) )
	\leq \beta(d_{\Xc}(x,x'),n).
\end{equation*}
Let us call the state-space system {\bfi state-uniformly (uniformly) AIS} if the function $\beta$ can be taken independently of $x,x'$ (and $\Seq{u}$) and satisfies in addition $\limsup_{t \rightarrow \infty} \sup_{r \in [0,s]} \beta(r,t) = 0$.
With these definitions, the AIS implies the SFP, and the state-uniform (uniform) SFP implies the state-uniform (uniform) AIS.\footnote{
To see that the state-uniform (uniform) SFP implies the state-uniform (uniform) AIS, consider $\beta$ defined on points $(s,N) \in [0,\infty) \times \N$ as follows and extend it to $[0,\infty) \times [0,\infty) \backslash \N$ in a way that yields a function of class $\Lc$ (in the state-uniform case one drops the supremum over $\Seq{u}$);
\begin{equation*}
	\beta(s,N)
	:= \sup_{n \geq N} \sup_{\Seq{u} \in \US} \sup_{\substack{x,x' \in \Xc\\d_{\Xc}(x,x')=s}} d_{\Xc}( \psi_n(x,\Seq{u}) , \psi_n(x',\Seq{u}) ).
\end{equation*}}
Conversely, if the state space $\Xc$ is bounded, then the state-uniform (uniform) AIS also implies the state-uniform (uniform) SFP.

We caution that the exact definitions of the AIS and the uniform AIS vary across sources.
The order of the quantifiers in the definitions should be checked carefully in each work.
For example, in \cite{JungersShakibWouw2024}, it is required that the inequality
\begin{equation*}
	d_{\Xc}( \psi_n(x,\sigma^m(\Seq{u})) , \psi_n(x',\sigma^m(\Seq{u})) )
	\leq \beta(d_{\Xc}(x,x'),n)
\end{equation*}
holds with $\beta$ possibly dependent on $\Seq{u} \in \US$ but independent of $m \in \Z$.
Furthermore, the function $\beta$ is sometimes required to be of class $\Kc$ as well, that is, for any fixed $t \in [0,\infty)$, the function $\beta(\cdot,t)$ needs to be continuous, strictly increasing, and satisfy $\beta(0,t) = 0$.

\subsection{Further equivalences}

In a similar spirit to the previous identifications, the next lemma highlights equivalent characterizations of the (shifted) SFP and IFP in the presence of the ESP.

\begin{lemma}
\label{lem_IFP}
	Suppose the state-space system has the ESP, and let $\Hc \colon \US^- \rightarrow \Xc$ be the induced functional that satisfies $\Hc(\Seq{u}) = \seq{x}{0}$ for all solutions $(\Seq{x},\Seq{u}) \in \XS^- \times \US^-$.
Then, the state-space system has
\begin{enumerate}[\upshape (i)]\itemsep=0em
\item\label{item_IFP_GO}
the (uniform) IFP if and only if for all $\Seq{u},\Seq{u}' \in \US$ with $\seq{u}{t} = \seq{u}{t}')$, $t \in \N$, (uniformly in $\Seq{u},\Seq{u}'$)
\begin{equation*}
	\limsup_{n \rightarrow \infty} d_{\Xc}( \Hc \circ \tau \circ \sigma^n(\Seq{u}) , \Hc \circ \tau \circ \sigma^n(\Seq{u}') ) = 0.
\end{equation*}

\item\label{item_IFP_Jaeger}
the (state-uniform) [uniform] s-IFP if and only if for all $\Seq{u},\Seq{u}',\Seq{u}'' \in \US^-$ (uniformly in $\Seq{u}',\Seq{u}''$) [uniformly in $\Seq{u},\Seq{u}',\Seq{u}''$]
\begin{equation*}
	\limsup_{n \rightarrow \infty} d_{\Xc}( \Hc \circ \gamma^n(\Seq{u}',\Seq{u}) , \Hc \circ \gamma^n(\Seq{u}'',\Seq{u}) ) = 0.
\end{equation*}

\item\label{item_fUAP}
the (state-uniform) [uniform] SFP if and only if for all $\Seq{u} \in \US$ and $x \in \Xc$ (uniformly in $x$) [uniformly in $\Seq{u},x$]
\begin{equation*}
	\limsup_{n \rightarrow \infty} d_{\Xc}( \psi_n(x,\Seq{u}) , \Hc \circ \tau \circ \sigma^n(\Seq{u}) ) = 0.
\end{equation*}

\item\label{item_UAP}
the state-uniform (uniform) s-SFP if and only if for all $\Seq{u} \in \US$ and $x \in \Xc$ uniformly in $x$ (uniformly in $\Seq{u},x$)
\begin{equation*}
	\limsup_{n \rightarrow \infty} d_{\Xc}( \psi_n(x,T^n(\Seq{u})) , \Hc \circ \tau(\Seq{u}) ) = 0.
\end{equation*}
\end{enumerate}
\end{lemma}

\begin{proof}
	\eqref{item_IFP_GO}
This was shown in \cref{sec_GrOr}.

\noindent\eqref{item_IFP_Jaeger}
Since $\Hc \circ \gamma^n(\Seq{u}',\Seq{u})$ is the unique state that is end-compatible with $\gamma^n(\Seq{u}',\Seq{u})$, this item is shown as in the last bullet point in \cref{sec_Ja}.

\noindent\eqref{item_fUAP}
One direction of the equivalence follows from the triangle inequality and the other one from the fact that $\Hc \circ \tau \circ \sigma^n(\Seq{u}) = \psi_n(\Hc \circ \tau(\Seq{u}),\Seq{u})$.

\noindent\eqref{item_UAP}
As for \eqref{item_fUAP}, one direction of the equivalence follows from the triangle inequality and the other one from the fact that $\Hc \circ \tau(\Seq{u}) = \psi_n(\seq{x}{-n},T^n(\Seq{u}))$.
However, note that the equivalence does not remain valid without the attribute `state-uniform' because the argument $\seq{x}{-n}$ of $\psi_n$ depends on $n$.
\end{proof}

Using this characterization, we present one final result.
Suppose the ESP holds.
\cref{lem_ESP_SFP} guarantees the state-uniform s-SFP if $\XS^-$ is compact.
Since the FMP is also implied by compactness of $\XS^-$ in the presence of the ESP, one can use the FMP instead as a weaker assumption.
In this case, we still get the state-uniform s-IFP.
Here, it is crucial that we used the product topology to define the FMP.

\begin{proposition}
	Suppose $\Uc$ is metrizable and the state-space system has the ESP and the FMP.
Then, the state-space system has the state-uniform s-IFP.
\end{proposition}

\begin{proof}
	Let $d_{\Uc}$ be a metric that metrizes the topology on $\Uc$.
Then, the product topology on $\US^-$ is metrized by $d_{\US^-}(\Seq{u}',\Seq{u}) = \sup_{t \in \Z_-} 2^t \min\{ 1 , d_{\Uc}(\seq{u}{t}',\seq{u}{t}) \}$.
In particular, $d_{\US^-}(\gamma^n(\Seq{u}',\Seq{u}),\Seq{u}) \leq 2^{-(n+1)}$ for any $\Seq{u},\Seq{u}' \in \US^-$ and $n \in \N$.
Now, given any $\Seq{u} \in \US^-$ and $\epsilon > 0$, the FMP yields some $\delta > 0$ such that $d_{\Xc}(\Hc(\Seq{u}'),\Hc(\Seq{u})) < \epsilon/2$ for all $\Seq{u}' \in \US^-$ that satisfy $d_{\US^-}(\Seq{u}',\Seq{u}) < \delta$.
Take $N \in \N$ so that $2^{-(N+1)} < \delta$.
Then, $d_{\Xc}(\Hc \circ \gamma^n(\Seq{u}',\Seq{u}),\Hc(\Seq{u})) < \epsilon/2$ for any $\Seq{u}' \in \US^-$ and $n \geq N$.
\cref{lem_IFP}.\eqref{item_IFP_Jaeger} and the triangle inequality yield the state-uniform s-IFP.
\end{proof}

\section{Conclusion}

This work has provided a unified framework for understanding the various notions of memory in RNNs, including steady states, echo states, state forgetting, input forgetting, and fading memory.
By clarifying the relationships between these concepts, we have shed light on the intricate dynamics of RNNs and their memory capacities.
Our results have established new implications and equivalences between these notions, providing a deeper understanding of how RNNs create reliable input/output responses, and have provided shorter proofs for previous related results.
Future research can build upon this foundation to further explore the properties and limitations of RNNs, ultimately leading to more efficient and effective models for a wide range of applications.

\acks{%
The authors acknowledge partial financial support from the School of Physical and Mathematical Sciences of the Nanyang Technological University.
The second author was partially funded by an Eric and Wendy Schmidt AI in Science Postdoctoral Fellowship at the Nanyang Technological University.
}

\bib{acm}{bibfile_FR}


\begin{thebibliography}{10}

\bibitem{AmigoEtal2024}
{\sc Amig{\'o}, J.~M., Dale, R., King, J.~C., and Lehnertz, K.}
\newblock {Generalized synchronization in the presence of dynamical noise and
  its detection via recurrent neural networks}.
\newblock {\em Chaos: An Interdisciplinary Journal of Nonlinear Science 34}, 12
  (Dec 2024), 123156.

\bibitem{Angeli2002}
{\sc Angeli, D.}
\newblock {A Lyapunov approach to incremental stability properties}.
\newblock {\em IEEE Transactions on Automatic Control 47}, 3 (2002), 410--421.

\bibitem{Ashby1956}
{\sc Ashby, W.~R.}
\newblock {\em {An Introduction to Cybernetics}}.
\newblock Chapman and Hall, 1956.

\bibitem{AstromWittenmark2008}
{\sc {\AA}str{\"o}m, K.~J., and Wittenmark, B.}
\newblock {\em {Adaptive Control}}, 2~ed.
\newblock Dover Books on Electrical Engineering. Dover Publications, 2008.

\bibitem{BogmansEtal2025}
{\sc Bogmans, C., Gomez-Gonzalez, P., Ganpurev, G., Melina, G., Pescatori, A.,
  and Thube, S.}
\newblock {Power Hungry: How AI Will Drive Energy Demand}.
\newblock {\em IMF Working Papers 2025}, 81 (4 2025), 32.

\bibitem{BoydChua1985}
{\sc Boyd, S., and Chua, L.}
\newblock {Fading memory and the problem of approximating nonlinear operators
  with Volterra series}.
\newblock {\em IEEE Transactions on Circuits and Systems 32}, 11 (1985),
  1150--1161.

\bibitem{CeniEtal2020PhysD}
{\sc Ceni, A., Ashwin, P., Livi, L., and Postlethwaite, C.}
\newblock {The echo index and multistability in input-driven recurrent neural
  networks}.
\newblock {\em Physica D: Nonlinear Phenomena 412\/} (2020), 132609.

\bibitem{ChuaGreen1976}
{\sc Chua, L., and Green, D.}
\newblock {A qualitative analysis of the behavior of dynamic nonlinear
  networks: Steady-state solutions of nonautonomous networks}.
\newblock {\em IEEE Transactions on Circuits and Systems 23}, 9 (1976),
  530--550.

\bibitem{DatarEtal2002}
{\sc Datar, M., Gionis, A., Indyk, P., and Motwani, R.}
\newblock {Maintaining Stream Statistics over Sliding Windows}.
\newblock {\em SIAM Journal on Computing 31}, 6 (2002), 1794--1813.

\bibitem{Diebold2024}
{\sc Diebold, F.~X.}
\newblock {Forecasting}, 2024.
\newblock reprint of the second edition published by Addison-Wesley.

\bibitem{EfimovEtal1978}
{\sc Efimov, N.~V., Kolmogorov, A.~N., Millionshchikov, V.~M., and Rozov,
  N.~K.}
\newblock {In memory of Boris Pavlovich Demidovich (obituary)}.
\newblock {\em Russian Math. Surveys 33}, 2 (1978), 195--202.
\newblock English translation of Russian original in Uspekhi Mat. Nauk 33:2
  (1978), 169--174.

\bibitem{FernandoSojakka2003}
{\sc Fernando, C., and Sojakka, S.}
\newblock {Pattern Recognition in a Bucket}.
\newblock In {\em Advances in Artificial Life\/} (2003), W.~Banzhaf,
  J.~Ziegler, T.~Christaller, P.~Dittrich, and J.~T. Kim, Eds., Springer Berlin
  Heidelberg, pp.~588--597.

\bibitem{FujiiNakajima2017}
{\sc Fujii, K., and Nakajima, K.}
\newblock {Harnessing Disordered-Ensemble Quantum Dynamics for Machine
  Learning}.
\newblock {\em Phys. Rev. Appl. 8\/} (Aug 2017), 024030.

\bibitem{GhoshEtal2021}
{\sc Ghosh, S., Nakajima, K., Krisnanda, T., Fujii, K., and Liew, T.~C.}
\newblock {Quantum neuromorphic computing with reservoir computing networks}.
\newblock {\em Advanced Quantum Technologies 4}, 9 (2021), 2100053.

\bibitem{QRC5}
{\sc Gonon, L., Mart{\'{i}}nez-Pe{\~{n}}a, R., and Ortega, J.-P.}
\newblock {Feedback-drive recurrent quantum neural network universality}.
\newblock {\em arXiv:2506.16332v1\/} (2025).

\bibitem{RC26}
{\sc Grigoryeva, L., Hamzi, B., Kemeth, F.~P., Kevrekidis, Y., Manjunath, G.,
  Ortega, J.-P., and Steynberg, M.~J.}
\newblock {Data-driven cold starting of good reservoirs}.
\newblock {\em Physica D: Nonlinear Phenomena 469\/} (2024), 134325.

\bibitem{RC18}
{\sc Grigoryeva, L., Hart, A.~G., and Ortega, J.-P.}
\newblock {Chaos on compact manifolds: Differentiable synchronizations beyond
  the Takens theorem}.
\newblock {\em Physical Review E - Statistical Physics, Plasmas, Fluids, and
  Related Interdisciplinary Topics 103\/} (2021), 062204.

\bibitem{RC7}
{\sc Grigoryeva, L., and Ortega, J.-P.}
\newblock {Echo state networks are universal}.
\newblock {\em Neural Networks 108\/} (2018), 495--508.

\bibitem{RC9}
{\sc Grigoryeva, L., and Ortega, J.-P.}
\newblock {Differentiable reservoir computing}.
\newblock {\em Journal of Machine Learning Research 20}, 179 (2019), 1--62.

\bibitem{Hamilton1994}
{\sc Hamilton, J.~D.}
\newblock {\em {Time Series Analysis}}.
\newblock Princeton University Press, 1994.

\bibitem{HartHookDawes2020}
{\sc Hart, A., Hook, J., and Dawes, J.}
\newblock {Embedding and approximation theorems for echo state networks}.
\newblock {\em Neural Networks 128\/} (2020), 234--247.

\bibitem{Haykin2014}
{\sc Haykin, S.}
\newblock {\em {Adaptive Filter Theory}}, 5~ed.
\newblock Pearson, 2014.

\bibitem{Jaeger2010}
{\sc Jaeger, H.}
\newblock {The ``echo state'' approach to analysing and training recurrent
  neural networks -- with an Erratum note}.
\newblock Tech. Rep. GMD Report 148, German National Research Center for
  Information Technology, 2010.

\bibitem{JiangSonneEtal2024}
{\sc Jiang, P., Sonne, C., Li, W., You, F., and You, S.}
\newblock {Preventing the Immense Increase in the Life-Cycle Energy and Carbon
  Footprints of LLM-Powered Intelligent Chatbots}.
\newblock {\em Engineering 40\/} (2024), 202--210.

\bibitem{JungersShakibWouw2024}
{\sc Jungers, M., Shakib, M.~F., and {van de Wouw}, N.}
\newblock {Discrete-Time Convergent Nonlinear Systems}.
\newblock {\em IEEE Transactions on Automatic Control 69}, 10 (2024),
  6731--6745.

\bibitem{Kalman1960}
{\sc Kalman, R.~E.}
\newblock {A new approach to linear filtering and prediction problems}.
\newblock {\em Journal of Basic Engineering 82}, 1 (1960), 35--45.

\bibitem{KalmanBertram1960a}
{\sc Kalman, R.~E., and Bertram, J.~E.}
\newblock Control system analysis and design via the ``second method'' of
  lyapunov: I--continuous-time systems.
\newblock {\em Journal of Basic Engineering 82}, 2 (Jun 1960), 371--393.

\bibitem{KalmanBertram1960b}
{\sc Kalman, R.~E., and Bertram, J.~E.}
\newblock {Control System Analysis and Design Via the ``Second Method'' of
  Lyapunov: II--Discrete-Time Systems}.
\newblock {\em Journal of Basic Engineering 82}, 2 (Jun 1960), 394--400.

\bibitem{KalmanBucy1961}
{\sc Kalman, R.~E., and Bucy, R.~S.}
\newblock {New Results in Linear Filtering and Prediction Theory}.
\newblock {\em Journal of Basic Engineering 83}, 1 (Mar 1961), 95--108.

\bibitem{KobayashiTranNakajima2024}
{\sc Kobayashi, S., Tran, Q.~H., and Nakajima, K.}
\newblock {Extending echo state property for quantum reservoir computing}.
\newblock {\em Phys. Rev. E 110\/} (Aug 2024), 024207.

\bibitem{KocarevParlitz1996}
{\sc Kocarev, L., and Parlitz, U.}
\newblock {Generalized Synchronization, Predictability, and Equivalence of
  Unidirectionally Coupled Dynamical Systems}.
\newblock {\em Phys. Rev. Lett. 76\/} (Mar 1996), 1816--1819.

\bibitem{LevitanPapush1966}
{\sc Levitan, B.~M., and Papush, P.~N.}
\newblock {Boris Pavlovich Demidovich (on his sixtieth birthday)}.
\newblock {\em Russian Math. Surveys 21}, 6 (1966), 159--164.
\newblock English translation of Russian original in Uspekhi Mat. Nauk 21:6
  (1966), 155--160.

\bibitem{Ljung1999}
{\sc Ljung, L.}
\newblock {\em {System Identification: Theory for the User}}, 2~ed.
\newblock Pearson Information and System Sciences Series. Prentice Hall, 1999.

\bibitem{LuHuntOtt2018Chaos}
{\sc Lu, Z., Hunt, B.~R., and Ott, E.}
\newblock {Attractor reconstruction by machine learning}.
\newblock {\em Chaos 28}, 6 (2018).

\bibitem{MaassNatschMarkram2002}
{\sc Maass, W., Natschläger, T., and Markram, H.}
\newblock {Real-Time Computing Without Stable States: A New Framework for
  Neural Computation Based on Perturbations}.
\newblock {\em Neural Computation 14}, 11 (Nov 2002), 2531--2560.

\bibitem{Manjunath2020ProcA}
{\sc Manjunath, G.}
\newblock {Stability and memory-loss go hand-in-hand: three results in dynamics
  and computation}.
\newblock {\em Proceedings of the Royal Society A 476\/} (2020), 20200563.

\bibitem{Manjunath2022Nonlin}
{\sc Manjunath, G.}
\newblock {Embedding information onto a dynamical system}.
\newblock {\em Nonlinearity 35}, 3 (Jan 2022), 1131.

\bibitem{ManjunathJaeger2013}
{\sc Manjunath, G., and Jaeger, H.}
\newblock {Echo State Property Linked to an Input: Exploring a Fundamental
  Characteristic of Recurrent Neural Networks}.
\newblock {\em Neural Computation 25}, 3 (2013), 671--696.

\bibitem{QRC1}
{\sc Mart\'{\i}nez-Pe\~na, R., and Ortega, J.-P.}
\newblock {Quantum reservoir computing in finite dimensions}.
\newblock {\em Phys. Rev. E 107\/} (Mar 2023), 035306.

\bibitem{QRC2}
{\sc Mart\'{\i}nez-Pe\~na, R., and Ortega, J.-P.}
\newblock {Input-dependence in quantum reservoir computing}.
\newblock {\em Phys. Rev. E 111\/} (Jun 2025), 065306.

\bibitem{Maybeck1982}
{\sc Maybeck, P.~S.}
\newblock {\em {Stochastic Models, Estimation, and Control}}, vol.~141-1:3 of
  {\em Mathematics in Science and Engineering}.
\newblock Academic Press, 1982.
\newblock 3 volumes.

\bibitem{Milnor1985}
{\sc Milnor, J.}
\newblock {On the concept of attractor}.
\newblock {\em Communications in Mathematical Physics 99}, 2 (1985), 177--195.

\bibitem{MujalEtal2021}
{\sc Mujal, P., Martínez-Peña, R., Nokkala, J., García-Beni, J., Giorgi,
  G.~L., Soriano, M.~C., and Zambrini, R.}
\newblock {Opportunities in Quantum Reservoir Computing and Extreme Learning
  Machines}.
\newblock {\em Advanced Quantum Technologies 4}, 8 (2021), 2100027.

\bibitem{Nakajima2020}
{\sc Nakajima, K.}
\newblock {Physical reservoir computing—an introductory perspective}.
\newblock {\em Japanese Journal of Applied Physics 59}, 6 (May 2020), 060501.

\bibitem{RC30}
{\sc Ortega, J.-P., and Rossmannek, F.}
\newblock {Fading Memory and the Convolution Theorem}.
\newblock {\em IEEE Transactions on Automatic Control 70}, 12 (2025),
  7830--7842.

\bibitem{RC28}
{\sc Ortega, J.-P., and Rossmannek, F.}
\newblock {State-space systems as dynamic generative models}.
\newblock {\em Proceedings of the Royal Society A 481}, 2309 (2025), 20240308.

\bibitem{RC31}
{\sc Ortega, J.-P., and Rossmannek, F.}
\newblock {Stochastic dynamics learning with state-space systems}.
\newblock {\em arXiv:2508.07876v1\/} (2025).
\newblock accepted for publication in Mathematical Models and Methods in
  Applied Sciences.

\bibitem{PavlovEtal2004}
{\sc Pavlov, A., Pogromsky, A., {van de Wouw}, N., and Nijmeijer, H.}
\newblock {Convergent dynamics, a tribute to Boris Pavlovich Demidovich}.
\newblock {\em Systems \& Control Letters 52}, 3 (2004), 257--261.

\bibitem{PavlovWouw2012}
{\sc Pavlov, A., and {van de Wouw}, N.}
\newblock {Steady-State Analysis and Regulation of Discrete-Time Nonlinear
  Systems}.
\newblock {\em IEEE Transactions on Automatic Control 57}, 7 (2012),
  1793--1798.

\bibitem{Sepulchre2021}
{\sc Sepulchre, R.}
\newblock {Fading Memory [From the Editor]}.
\newblock {\em IEEE Control Systems Magazine 41}, 1 (2021), 4--5.

\bibitem{ShalevShwartz2012}
{\sc Shalev-Shwartz, S.}
\newblock {Online Learning and Online Convex Optimization}.
\newblock {\em Foundations and Trends® in Machine Learning 4}, 2 (2012),
  107--194.

\bibitem{Simon2006}
{\sc Simon, D.}
\newblock {\em {Optimal State Estimation: Kalman, H{\(\infty\)}, and Nonlinear
  Approaches}}.
\newblock Wiley, 2006.

\bibitem{SuttonBarto2018}
{\sc Sutton, R.~S., and Barto, A.~G.}
\newblock {\em {Reinforcement Learning: An Introduction}}, 2~ed.
\newblock MIT Press Cambridge, 2018.

\bibitem{TanakaEtal2019}
{\sc Tanaka, G., Yamane, T., H{\'e}roux, J.~B., Nakane, R., Kanazawa, N.,
  Takeda, S., Numata, H., Nakano, D., and Hirose, A.}
\newblock {Recent advances in physical reservoir computing: A review}.
\newblock {\em Neural Networks 115\/} (2019), 100--123.

\bibitem{TranRuefferKellett2019TAC}
{\sc Tran, D.~N., R{\"u}ffer, B.~S., and Kellett, C.~M.}
\newblock {Convergence Properties for Discrete-Time Nonlinear Systems}.
\newblock {\em IEEE Transactions on Automatic Control 64}, 8 (2019),
  3415--3422.

\bibitem{Volterra1959}
{\sc Volterra, V.}
\newblock {\em {Theory of Functionals and of Integral and Integro-Differential
  Equations.}}
\newblock Dover, 1959.

\bibitem{Wiener2019}
{\sc Wiener, N.}
\newblock {\em {Cybernetics or Control and Communication in the Animal and the
  Machine}}.
\newblock The MIT Press, 10 2019.
\newblock Reissue of the 1961 second edition.

\end{thebibliography}
\end{document}